\documentclass[conference]{IEEEtran}
\IEEEoverridecommandlockouts
\usepackage{cite}
\usepackage{amsmath,amssymb,amsfonts}
\usepackage{algorithmic}
\usepackage{graphicx}
\usepackage{textcomp}
\usepackage{xcolor}
\usepackage{amsthm}
\usepackage{algorithm}
\usepackage[hyphens]{url} 
\usepackage{graphicx}
\usepackage{booktabs}
\graphicspath{{figures/}}
\usepackage{caption}
\usepackage{subcaption}

\usepackage{amsmath,amsfonts,bm}

\def\eqref#1{equation~\ref{#1}}

\def\1{\bm{1}}

\DeclareMathAlphabet{\mathsfit}{\encodingdefault}{\sfdefault}{m}{sl}
\SetMathAlphabet{\mathsfit}{bold}{\encodingdefault}{\sfdefault}{bx}{n}

\DeclareMathOperator{\sign}{sign}

\newtheorem{theorem}{Theorem}
\usepackage{color}
\newcommand{\pa}{\textcolor{black}}

\def\BibTeX{{\rm B\kern-.05em{\sc i\kern-.025em b}\kern-.08em
    T\kern-.1667em\lower.7ex\hbox{E}\kern-.125emX}}
\begin{document}

\title{Online Probabilistic Model Identification Using Adaptive Recursive MCMC \thanks{Disclaimer: This work has been accepted for publication in the International Joint Conference on Neural Networks (IJCNN). © 2023 IEEE.  Personal use of this material is permitted.  Permission from IEEE must be obtained for all other uses, in any current or future media, including reprinting/republishing this material for advertising or promotional purposes, creating new collective works, for resale or redistribution to servers or lists, or reuse of any copyrighted component of this work in other works.}
}

\author{\IEEEauthorblockN{ Pedram Agand}
\IEEEauthorblockA{\textit{Department of Computer Science,} \\
\textit{Simon Fraser University}\\
Burnaby, Canada. \\
pagand@sfu.ca}
\and
\IEEEauthorblockN{ Mo Chen}
\IEEEauthorblockA{\textit{Department of Computer Science,} \\
\textit{Simon Fraser University}\\
Burnaby, Canada. \\
0000-0001-8506-3665}
\and
\IEEEauthorblockN{ Hamid D. Taghirad}
\IEEEauthorblockA{\textit{Department of Electrical Engineering,} \\
\textit{K. N. Toosi University of Technology}\\
Tehran, Iran. \\
0000-0002-0615-6730}
}

\maketitle

\begin{abstract}
Although the Bayesian paradigm offers a formal framework for estimating the entire probability distribution over uncertain parameters, its online implementation can be challenging due to high computational costs.
We suggest the Adaptive Recursive Markov Chain Monte Carlo (ARMCMC) method, which eliminates the shortcomings of conventional online techniques while computing the entire probability density function of model parameters.
The limitations to Gaussian noise, the application to only linear in the parameters (LIP) systems, and the persistent excitation (PE) needs  are some of these drawbacks.
In ARMCMC, a temporal forgetting factor (TFF)-based variable jump distribution is proposed. The forgetting factor can be presented adaptively using the TFF in many dynamical systems as an alternative to a constant hyperparameter. By offering a trade-off between exploitation and exploration, the specific jump distribution has been optimised towards hybrid/multi-modal systems that permit inferences among modes. These trade-off are adjusted based on parameter evolution rate. We demonstrate that ARMCMC requires fewer samples than conventional MCMC methods to achieve the same precision and reliability.
We demonstrate our approach using parameter estimation in a soft bending actuator and the Hunt-Crossley dynamic model, two challenging hybrid/multi-modal benchmarks. 
Additionally, we compare our method with recursive least squares and the particle filter, and show that our technique has significantly more accurate point estimates as well as a decrease in tracking error of the value of interest.
\end{abstract}

\begin{IEEEkeywords}
MCMC, Bayesian optimization, hybrid/multi-modal systems, temporal forgetting factor.
\end{IEEEkeywords}

\section{Introduction}
Bayesian methods are powerful tools to not only obtain a numerical estimate of a parameter but also to give a measure of confidence  \cite{bishop2006pattern}. \pa{This is  obtained by calculating the probability distribution of parameters rather than a point estimate, which is prevalent in frequentist paradigms} \cite{NIPS2018_8216,agand2017teleoperation}. 
One of the main advantages of probabilistic  frameworks is that they enable decision making under uncertainty. In addition, knowledge fusion is significantly facilitated in probabilistic  frameworks; different sources of data or observations can be combined according to their level of certainty in a principled manner \cite{agand2019adaptive}.
Using credible intervals instead of confidence intervals \cite{NIPS2019_8868}, an absence of over parameterized phenomena \cite{bishop2006pattern}, and evaluation in the presence of limited number of observed data \cite{joho2013nonparametric} are other distinct features of this framework. 

Nonetheless, Bayesian inference requires high computational effort for obtaining the whole probability distribution and requires prior  general knowledge about the noise distribution.
Consequently, strong simplifying assumptions were often made (e.g. calculating specific features of the model parameters distribution rather than the whole distribution). 
One of the most effective methods for Bayesian inferences  is  Markov Chain Monte Carlo  (MCMC).
In the field of system identification, MCMC variants such as the one proposed by \cite{green2015bayesian} are mostly focused on offline system identification.
This is partly due to computational challenges which prevent its real-time use. 
There are extensive research in the literature which investigate how to increase sample efficiency (e.g. \cite{nori2014r2,mandel2016efficient}). 
The authors in \cite{green1995reversible} first introduced reversible jump Markov chain Monte Carlo (RJMCMC) as a method to address the model selection problem. In this method, an extra pseudo-random variable is defined to address dimension mismatch. There are further extensions of MCMC in the literature; however, there is a lack of variants suitable for online estimation. One example which claims to have often much faster convergence is No-U-Turn Sampler (NUTS), which can adapt proposals ``on the fly'' \cite{hoffman2014no}. 
 
\pa{If the general form of the relation between inputs and outputs is known either by physical relation or knowing the basis function, parametric identification techniques are an effective way to model a system. To identify the parameters in a known model, researchers propose frequency (\cite{lin2019frequency}) or time domain (\cite{agand2022human}) approaches. Noisy measurements, inaccuracy, inaccessibility, and costs are typical challenges that limit direct measurement of unknown parameters in a physical/practical system  \cite{agand2016particle}. Parameter identification techniques have been extensively used in different subject areas including but not limited to chemistry, robotics, fractional models and health sectors \cite{yang2020comprehensive, kim2019robotic,yu2022life,agand2016transparent,houssein2021enhanced}.  For instance, motion filtering and force prediction in robotics applications are important fields of study with interesting challenges which makes them suitable test cases for Bayesian inferences \cite{saar2018model}. }

\pa{Different parametric  identification methods have been proposed in the literature for linear and Gaussian noise  \cite{pmlr-v80-wang18f}; however, in cases of  nonlinear hybrid/multi-modal systems (e.g. Hunt-Crossley or fluid soft bending actuator) or systems with non-Gaussian noise (e.g. impulsive disturbance), there is no optimal solution for the identification problem. In addition, in cases that the assumed model for the system is not completely valid,  MCMC implementation of Bayesian approach can provide reasonable inferences.}  Authors in \cite{wang2019parameter} utilize the  least square methods for nonlinear physical modeling of air dynamics in pneumatic soft actuator.   A method to determine the damping term in the Hunt-Crossley model was proposed in \cite{carvalho2019exact}.  A single-stage method for the estimation of the Hunt-Crossley model is proposed by \cite{haddadi2012real} which requires some restrictive conditions to calculate the parameters. \pa{ Moreover, the method does not offer any solution for discontinuity  in the dynamic model, which is common in the transition phase from contact to free motion and vice versa \cite{williams2018robust}.}

This paper proposes a new technique, Adaptive Recursive  Markov Chain Monte Carlo (ARMCMC), to address several weaknesses of traditional online identification methods, such as solely being applicable to systems Linear in Parameters (LIP), having Persistent Excitation (PE) requirements, and assuming Gaussian noise. 
ARMCMC is an online method that takes advantage of the previous posterior distribution whenever there is no abrupt change in the parameter distribution. 
To achieve this, we define a new \emph{variable jump distribution} that accounts for the degree of model mismatch using a \textit{temporal forgetting factor (TFF)}. 
The TFF is computed from a model mismatch index and determines whether ARMCMC employs modification or reinforcement to either restart or refine the estimated parameter distribution. As this factor is a function of the observed data rather than a simple user-defined constant, it can effectively adapt to the underlying dynamics of the system. We demonstrate our method using two different examples: a soft bending actuator and  the Hunt-Crossley model. We show favorable performance compared to the state-of-the-art baselines.


\section{PRELIMINARIES}
\subsection{Problem statement}
In the Bayesian paradigm,  parameter estimations are given in the form of the posterior probability density function (pdf); this pdf can be continuously updated as new data points are received.  
Consider the following general  model:
\begin{equation}
Y=F_j(X,\theta_j)+\nu_j, ~~\forall j \in \{1,2, \ldots, m\},
\label{eq:bays44444444444444}
\end{equation}
\noindent where $Y$, $X$,   $\theta$, and $\nu$   are concurrent output, input,  model parameters set and noise vectors, respectively. For a hybrid system, there can be $m$ different nonlinear functions ($F_j$) with different parameters set ($\theta_j$) for each mode. For multi-modal systems, there can be $m$ different noise distribution $\nu_j$ for each mode.  To calculate the posterior pdf, the observed data (input/output pairs) along with a prior distribution are combined via Bayes' rule \cite{khatibisepehr2013design}. 
The Bayesian rule is defined as follows \cite{bishop2006pattern}.
\begin{equation}
P(\theta|D)=\frac{P(D|\theta)P(\theta)}{P(D)},
\label{eq:bays}
\end{equation}

where, $P(\theta)$ is the prior probability of parameters, $P(D)=\int P(D|\theta)P(\theta)d\theta,$ is called as evidence,  $P(D|\theta)$ is the likelihood function,  and $P(\theta|D)$ is posterior probability of parameters given the data. 
We will be applying updates to the posterior pdf using batches of data points; hence, it will be convenient to partition the data as follows:
\begin{equation}
D^t=\{(X,Y)_{t_m+1},(X,Y)_{t_m+2},\cdots,(X,Y)_{t_m+N_s+1}\},
\label{eq:data}
\end{equation}
where $N_s=T_s/T$ is the number of data points in each data pack with $T, T_s$ being  the data  and  algorithm sampling times, respectively. This partitioning is convenient for online applications, as $D^{t-1}$ should have been previously collected so that the algorithm can be executed  from $t_{m}$ to $t_{m}+N_s+1$, an interval which we will define as  algorithm time step $t$. Ultimately, inferences are completed at $t_{m}+N_s+2$. Fig. \ref{fig:timetable} illustrates the timeline for the data and the algorithm. It is worth mentioning that the computation can be done simultaneously   with the task  in the adjacent algorithm step (e.g.  phase A of algorithm $t$, phase B of algorithm $t-1$ and phase C of algorithm $t-2$ can all be done simultaneously). 

 \begin{figure}[t!]
     \centering
    \centerline{\includegraphics[width=1\linewidth,trim={5cm  6cm 5cm 6cm },clip]{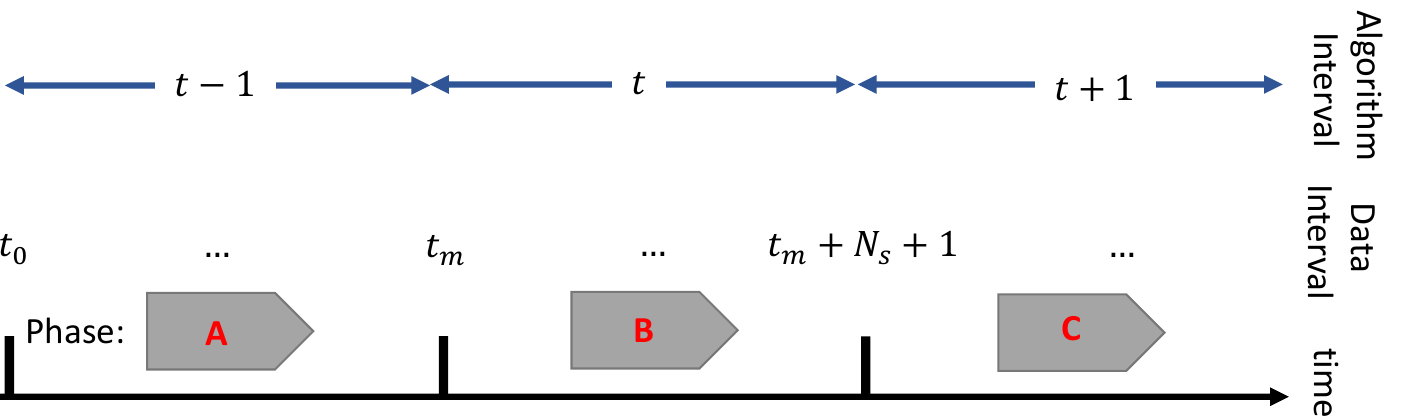}}
\caption{  Data timeline and different phases  of ARMCMC algorithm. For algorithm at time t: Phase (A) Data collection [ $N_s$ data points are packed for the next algorithm time step], Phase (B) Adjustment [the method is applying to the most recent  pack], (C) Execution [after the min evaluation of algorithm, the results will be updated on posterior distribution and any byproduct value of interest].}
 \label{fig:timetable}
   \end{figure}
 
According to Bayes' rule (\ref{eq:bays}) and assuming data points are independent and identically distributed ( i.i.d)  in   Eq. (\ref{eq:bays44444444444444}),  we have
\begin{equation}
P(\theta^t|[D^{t-1},D^t])=\frac{P\big(D^t|\theta^t,D^{t-1}\big)P(\theta^t|D^{t-1})}{\int P\big(D^{1}
|\tau^t,D^{t-1}\big)P(\tau^t|D^{t-1}) d\tau^t},
\label{eq:bys1}
\end{equation}
where $\theta^t$ denotes the parameters at current algorithm time steps. $P(\theta^t|D^{t-1})$ is the prior distribution over parameters, which is also the posterior distribution at the previous algorithm time step.
Probability $P\big(D^{t}
|\theta^t,D^{t-1}\big)$ is the likelihood function which is obtained by sampling from the one-step-ahead prediction:
\begin{equation}
\hat{Y}^{t|t-1}=F(D^{t-1},\theta^{t}),
\label{eq:bys22}
\end{equation}
where $\hat{Y}^{t|t-1}$ is a sample from the prediction of the output in (\ref{eq:bays44444444444444}). If the model in (\ref{eq:bys22}) is accurate, then the difference between the real
output and predicted is the measurement noise, (i.e.,
$Y^{t|t-1}-\hat{Y}^{t|t-1}=\nu$). Therefore, the model parameter may be updated as follows:
\begin{equation}
P\big(D^t|\theta^t,
D^{t-1}\big)=\prod_{t=t_m+1}^{t_m+N_s+1}
P_\nu\big(Y^{t|t-1}-\hat{Y}^{t|t-1}\big),
\label{eq:bys3}
\end{equation}
where $P_\nu$ is the probability distribution of noise. Note that there is no restriction on the type of noise probability distribution.  A good approximation can be a Gaussian distribution with  sample means and variances of the data as its parameters.


\subsection{Markov Chain Monte Carlo (MCMC)}
MCMC is often employed to compute the  posterior pdf numerically. The multidimensional integral in (\ref{eq:bys1}) is approximated by samples  drawn from the posterior pdf. The samples are first drawn from a different distribution called proposal distribution, denoted $q(\cdot)$, which can be sampled more easily compared to the posterior.
The main steps of the Metropolis-Hastings algorithm are listed as follows \cite{ninness2010bayesian}:

\begin{enumerate}
\item Set initial guess  $\theta_0$ while $P(\theta_0|Y)> 0$  for  $k=1$,
\item Draw candidate  parameter $\theta_{cnd}$, at iteration $k$, from the proposal distribution, $q(\theta_{cnd}|\theta_{k-1})$
\item Compute the acceptance probability,
\begin{equation}
\alpha(\theta_{cnd}|\theta_{k-1})=
\min\Big\{ 1,\frac{P(\theta_{cnd}|D)q(\theta_{k-1}|\theta_{cnd})}{P(\theta_{k-1}|D)q(\theta_{cnd}|\theta_{k-1})}\Big\},
\label{eq:accpt}
\end{equation}
\item Generate a uniform random number $\gamma$ in $[0,1]$,
\item ``Accept'' candidate if $\gamma\leq\alpha$ and ``ignore'' it if $\gamma>\alpha$,
\item Set iteration to $k + 1$ and go to step 2.
\end{enumerate}

\subsection{Precision and reliability}
Two important notions in probabilistic frameworks to compare results are precision ($\epsilon$) and reliability ($\delta$). The former represents the proximity  of  a sample to the ground truth, and the latter represents the probability that an accepted sample lies within $\epsilon$ of the ground truth. 

\emph{Lemma:} \pa{Let $P_k$ to be the probability of $k$ samples from MCMC, and $\mathbb{E} (P_{k})$ to denote their expected value. According to Chernoff bound, to achieve the required precision and reliability  $\epsilon, \delta \in [0,1]$,  as $ Pr\Big\{[P_{k} -\mathbb{E} (P_{k})]\leq \epsilon\Big\}\geq \delta$, the minimum number of samples ($k$) must satisfy the following relation \cite{tempo2012randomized}}
 \begin{equation}
 k\geq \frac{1}{2\epsilon^2} \log(\frac{2}{1-\delta}).
 \label{eq:chern}
 \end{equation}


\section{ARMCMC ALGORITHM}
\label{sec:armcmc}
At each algorithm time interval, ARMCMC recursively estimates the posterior pdf by drawing samples. The number of samples drawn is constrained by the desired precision and reliability, and the real-time requirement. On the other hand, the maximum number of data points in each data pack ($N_s$) is limited by the frequency of model variation, and the minimum is confined by the shortest required time, such that the algorithm is real-time. 

We propose a variable jump distribution that enables both exploiting and exploring. This will necessitate the definition of the TFF as a model mismatch measure to reflect current underlying dynamics of the data. We also prove that ARMCMC achieves the same precision and reliability with fewer samples compared to the traditional MCMC. Algorithm 1 summarizes ARMCMC.
   \vskip -.1in
\begin{algorithm}[!t]
   \caption{ARMCMC}
   \label{alg:example}
\begin{algorithmic}
\STATE \emph{Assumptions:} 1) roughly noise mean ($\mu_\nu$) 2) roughly noise variance ($\sigma_\nu$)  3) desired precision and reliability ($\epsilon_0,\delta_0$) 4) desired threshold for model mismatch ($\zeta_{th}$) \\
\STATE \emph{Goal:} Online calculation of parameters posterior distribution given the consecutive $t$-th pack of data ($P(\theta^t|D^t)$)
   \STATE {\bfseries Initialization:} Prior knowledge for $\theta^0_1$, $n=0$
   \STATE Consider desire precision and reliability ($\epsilon,\delta$)
   \REPEAT
   \STATE Put $t_0 = n*N_s+1$ from (\ref{eq:data}), $n++$
   \STATE Add new data pack to dataset $D^t$
   \STATE \emph{Model mismatch index:} $\zeta^t$ from (\ref{eq:bys5333})
   \IF{$\zeta^t<\zeta_{th}$}
   \STATE \emph{Reinforcement:}  set prior knowledge equal to the latest posterior of previous pack
   \STATE \emph{TFF:} $\lambda^t$ from (\ref{eq:bys522221})
   \ELSE
   \STATE \emph{Modification:} set prior knowledge $\theta^n_1$
   \STATE \textit{TFF:} $\lambda^t = 0$
   \ENDIF
   \STATE \emph{Set minimum iteration} $k_{min}$ from (\ref{eq:neval})
   \FOR{$k=1$ {\bfseries to} $k_{min}$}
   \STATE \emph{Proposal distribution:}
   \STATE ~~~ $\bullet$ draw $\lambda_k \sim U(0,1)$
   \STATE ~~~ $\bullet$ \emph{Variable jump distribution:}  $q_k^t(.)$ from (\ref{bys4})
   \STATE Draw $\theta^{t*}_k \sim q_k^t(.)$
   \STATE \emph{Acceptance rate:} $\alpha(.)$ from (\ref{eq:accpt})
   \STATE Draw $\gamma \sim U(0,1)$
   \IF{$\gamma\leq\alpha$}
   \STATE `Accept' the proposal 
   \ENDIF
   \ENDFOR
   \STATE {\bfseries Wait} to build  $D_{t_0}^{t_{m+N_s+1}}$ (algorithm sample time)
   \UNTIL{No data is obtained}
\end{algorithmic}
\end{algorithm}
\vskip -2in

\subsection{Variable jump distribution}
We propose a variable jump distribution (also known as a proposal distribution) to achieve faster convergence, thereby enabling real-time parameter estimation:
\begin{equation}
q_k^t(\theta^{t}|\theta_{k-1}^t)=
  \begin{cases}
    P(\theta^{t-1}|D^{t-1})       & \quad  \lambda_k \leq \lambda^t\\
   N(\mu_D,\sigma_\nu) & \quad  \lambda_k > \lambda^t
  \end{cases},
\label{bys4}
\end{equation}
where $\theta_{k-1}^t$ is the $(k-1)$-th parameter sample which is given by the $t$-th data pack throughout the MCMC evaluation. In each algorithm time sample, the averages of the second half of this quantity will construct $\theta^{t}$. 
$P(\theta^{t-1}|D^{t-1})$ is the posterior pdf of
the parameters at the previous algorithm time step, and $N(\mu_D,\sigma_\nu)$ is a Gaussian distribution with its mean and variance $\mu_D, \sigma_\nu$  computed using the empirical mean and variance of $D^{t-1}$ \pa{ if available. $\lambda_k$ is drawn from a uniform distribution between 0,1 to resemble a randomize algorithm.}

\pa{The hyperparameter $\lambda^t$  is called \textit{TFF} which is an adaptive threshold for the  $t$-th pack that takes inspiration from  classical system identification. This will regulate how previous knowledge affects the posterior pdf. Intuitively, a smaller value for $\lambda^t$  means that  there might be a larger/sudden change in the ground truth value, and thus more exploration is needed.} Conversely, a larger value of $\lambda^t$ is appropriate when $\theta$ is changing slowly, and therefore previous knowledge should be exploited. Exploiting this knowledge will lead to an overall better precision and reliability. 

\emph{Remark 1:} \pa{Note that according to \cite{bishop2006pattern}, ARMCMC is still a valid MCMC scheme since the variable jump policy does not violate the independency between samples. The reason is that the decision to  whether select the proposal from previous samples or not is purely stochastic. The only difference is proposing more reasonable candidates with an adaptive threshold of acceptance. Further elaboration will be presented in following sections.}

\subsection{Temporal forgetting factor (TFF)}
Depending on the changes in the distribution of the parameter $\theta$, a new sample can be drawn according to the \textit{modification} or the \textit{reinforcement} mode.  Modification is employed  to re-identify the distribution `from scratch' while reinforcement is employed to make the identified probability distribution more precise when it is not undergoing sudden change. \pa{To  simulating from the prior distribution, we use the inverse CDF method \cite{li2021improving}}. Therefore, we define a \textit{model mismatch index}, denoted $\zeta^t$, such that when it surpasses a predefined threshold ($\zeta^t > \zeta_{\text{th}}$), modification  is applied. Otherwise $\zeta^t$ is used to determine $\lambda^t$ as follows:
\begin{equation}
\lambda^t=e^{-|\mu_\nu-\zeta^t|},
\label{eq:bys522221}
\end{equation}
where $\mu_\nu$ is an estimation of the noise mean, based on the calculation of the expected value in Eq.  (\ref{eq:bays44444444444444}). Note that employing modification is equivalent to setting $\lambda^t = 0$. The model mismatch index $\zeta^t$ itself is calculated by averaging  the errors of the previous model given the current data:
\begin{equation}
\zeta^t=1/N_s\sum_{n =1}^{N_s} \Big(y_n^t-\mathop{\mathbb{E}}_{\theta \in \theta^{t-1}} (F(D^t(n),\theta))\Big), \zeta^0 = \infty
\label{eq:bys5333}
\end{equation}

\emph{Remark 2:} The model mismatch index accounts for all sources of uncertainty in the system. To calculate $\zeta_{th}$, one needs to precalculate the persisting errors between the predicted and measured data. In other words, $\zeta_{th}$  is an upper bound for the unmodeled dynamics, disturbances, noises, and any other sources of uncertainty in the system, which can be estimated via offline learning given the model capacity.

\emph{Remark 3:} To avoid numerical issues, the summation of probability logarithms are calculated. In addition, each data pair in the algorithm time sample is weighted based on its temporal distance to the current time. Therefore Eq. (\ref{eq:bys3}) is modified as:
\begin{equation}\begin{split}
&\log \big( P(\cdot) \big)=\sum_{t_m+1}^{t_m+N_s+1} \log P_\nu
(e^t),\\
e^t=&\Big(Y^{t}_n-F^t(D^{t-1}(n),\theta^{t})\Big)e^{-\rho(N_s-n)},
\label{eq:g1}
\end{split}\end{equation}
where  $\rho \in [0,1]$ is a design parameter that reflects the volatility of the model parameters, and $e^t=[e^t_1, ...,e^t_n,...,e^t_{N_s}]$. For systems with fast-paced parameters, $\rho$ should take larger values.

\subsection{Minimum required evaluation}
\label{sec:theorem}
\begin{theorem}
Let $\epsilon$ and $\delta$ be the desired precision and reliability. Furthermore, assume that the initial sample has an enough number of evaluations.  To satisfy the inequality in Eq. (\ref{eq:chern}),  the minimum number of samples $k$ in ARMCMC is calculated using this implicit relation:
\begin{equation}
k_{min} = \frac{1}{2\epsilon^2} \log(\frac{2}{\lambda^t(1-\delta)+2(1-\lambda^t)e^{-2\epsilon^2(1-\lambda^t)k_{min}}}).
\label{eq:neval}
\end{equation}
\end{theorem}

\begin{proof}
\textit{Samples from previous pdf:} According to the variable jump distribution in (\ref{bys4}), given $k$ samples, the expected number of samples drawn from the previous posterior pdf ($P(\theta|D^{t})$) is $\lambda^t k$. 
By assumption, the algorithm has already drawn at least $k$ samples in the previous algorithm time-step.
Consequently,  by (\ref{eq:chern}), the expected number of samples with distances less than $\epsilon$ from $\mathbb{E}(P_k)$ are drawn from a previous distribution of at least $\lambda k \delta$.



\textit{Samples from Gaussian:} By (\ref{bys4}), there are $k_0 =(1-\lambda^t) k$ samples drawn in expectation. 
Based on this assumption, we have $Pr\Big\{\{P_{k} -\mathbb{E} (P_{k})\}\leq \epsilon\Big\}\geq \delta_0$, where $\delta_0$ is given by rearranging (\ref{eq:chern}):

\begin{equation}
    \delta_0 = 1-2 e^{-2\epsilon^2k_0}.
    \label{eq:delta}
\end{equation}

Thus, the expected number of samples with distances less than $\epsilon$ from $\mathbb{E}(P_k)$ are at least $\delta_0 (1-\lambda^t) k$.

\begin{figure}[t!]
     \centering
     \centerline{\includegraphics[width=1\linewidth,trim={1.5cm  0.2cm 1.5cm 0.5cm },clip]{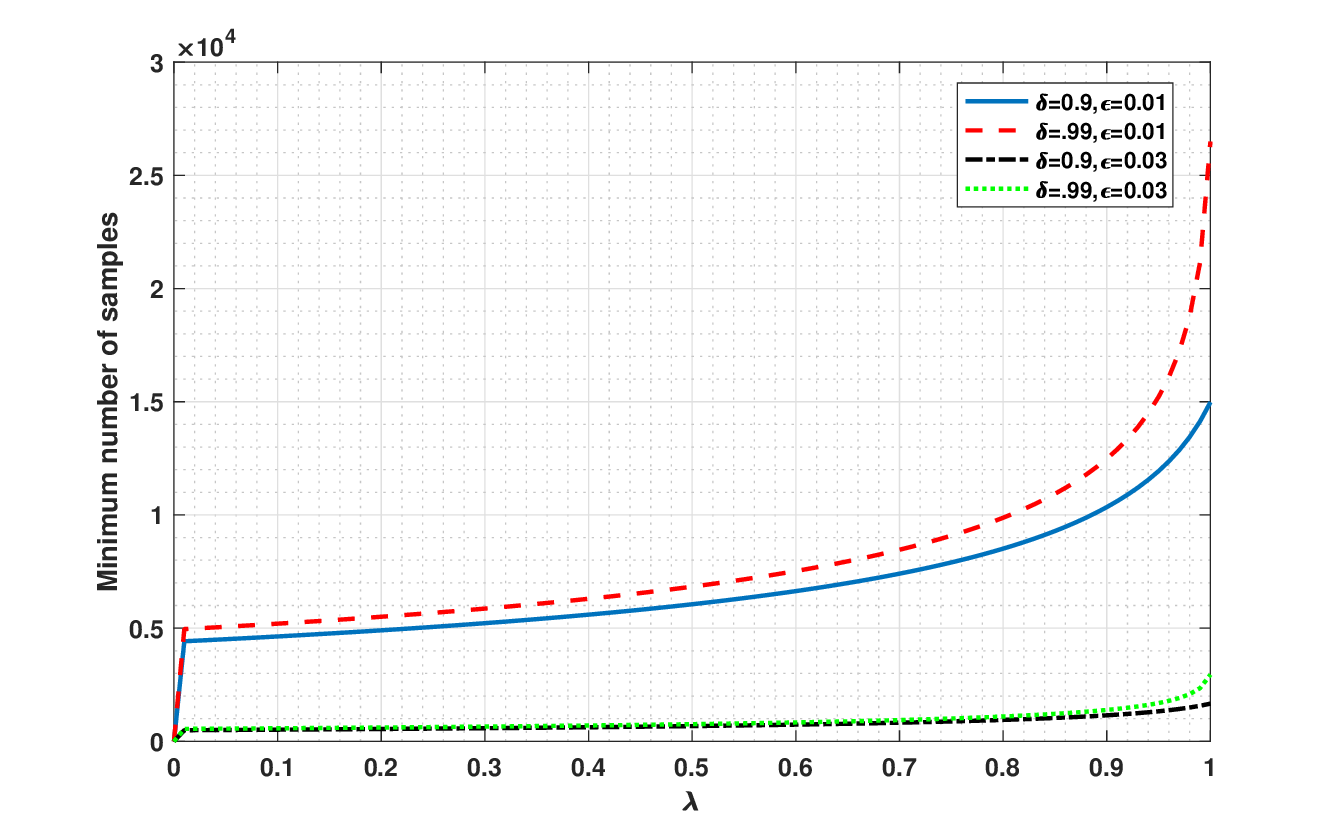}}
\caption{$K_{min}$ with respect to $\lambda$ for some values of $\epsilon,\delta$ in ARMCMC. (for $\lambda=1$ evaluation for ARMCMC is equivalent to MCMC)}
 \label{fig:lambda}
\end{figure}

\textit{Overall reliability:} The total expected number of samples with distances less than $\epsilon$ from $\mathbb{E}(P_k)$ are the summation of the two parts mentioned above. Hence, it is obtained through  dividing by $k$:
\begin{align}
    \delta_1 &= \frac{(\lambda^t k\delta) + (\delta_0 (1-\lambda^t) k)}{k} 
    \label{eq:proof}
\end{align}
\pa{With the new obtained reliability, which is greater than the desired one, we can safely decrease  the number of evaluations.}  For the sake of illustration, Fig. \ref{fig:lambda} presents the minimum required number of evaluations with respect to $\lambda$ for different precisions and reliabilities. As it can be seen, the ARMCMC is equal to MCMC if  $\lambda$ is always set to one, and have trivial solution for $\lambda=0$. The number of evaluations in ARMCMC mitigates as the validity of the previous model increases. 
\end{proof}

\section{RESULTS}
\label{sec:result}
In this section, we demonstrate the performance of the proposed approach with two examples. First, we employ ARMCMC to identify parameters in the soft bending actuator model and compare the results by Recursive Least Square (RLS) and Particle Filter (PF). RLS is the recursive version of the least squares solution for online data, which utilizes the matrix inversion lemma and Riccati equation as a special case of the Kalman filter \cite{simon2006optimal}. Sequential Monte Carlo methods, also known as PF, comprise a collection of Monte Carlo algorithms that address filtering problems encountered in Bayesian statistical inference and signal processing \cite{kantas2015particle}.  The goal is to use a sequence of data to recursively characterise the joint posterior distributions of input and output given the state transition density  \cite{pitt1999filtering}. In the second example, we evaluate  our method on the Hunt-Crossley model when given a reality-based simulation, and compare it with a simple MCMC and RLS. All the results/code are available on \url{https://github.com/pagand/ARMCMC}.

\subsection{Fluid soft bending actuator} 
Consider the dynamic model of a fluid soft bending actuator, given by \cite{wang2019parameter}:
\begin{equation}\begin{split}
\ddot{\alpha} = q_1(p-p_{atm})&-q_2 \dot{\alpha}-q_3 \alpha, \\
u_c \sign(p_s - p)\sqrt{|p_s-p|} =& q_4 \dot{p}+q_5 \dot{p}p, u_d =0,\\
 u_d \sign(p - p_{atm})\sqrt{|p-p_{atm}|} &= q_6 \dot{p}+q_7 \dot{p}p, u_c=0,
\label{eq:dyn_2end}
\end{split}
\end{equation}
where $\alpha$ is the angle of the actuator and $u_c, u_d$ are  the control inputs for retraction and contraction. Also, $p,p_s,p_{atm}$ are the current, compressor and atmosphere pressure respectively. For this example, we assume $q_1 = 1408.50, q_2 =132.28, q_3 = 3319.40$ are known and $p_{atm} = 101.3$ kPa, $p_s = 800$ kPa. We are trying to identify the remaining four  parameters ($q_4,...,q_7$). To this end, we assume the  hybrid model below:
\begin{equation}
u ~\sign(\Delta p)\sqrt{|\Delta p|} = \theta_1 \dot{p}+\theta_2 \dot{p}p, u=\{u_c, u_d\},
    \label{eq:dyn_2nd2}
\end{equation}
where $\Delta p$ is either $(p_s-p)$ for retraction or $(p-p_{atm})$ for contraction, and $\theta_1$ represents $q_4,q_6$, while $\theta_2$ represents $q5,q7$. For RLS, as the range of parameters are small,  we scale the input vector by a factor of $10^7$. Given the input ($u_c, u_d$) and the output ($p,\dot{p}$), we want to identify the parameter and estimate the current angle of the actuator, assuming its initial position is at the origin. 

\subsubsection{Modeling for PF}
First, we  derive the first order differential relation as follows:
\begin{equation}\begin{split}
    \dot{x}_2 =& q_1(x_3-p_{atm}) -q_2 x_2 - q_3 x_1 \\
    \dot{x}_3 =& \frac{u~ \sign(\Delta p)\sqrt{\Delta p}}{x_4+x_5 x_3}
\end{split}
\label{eq:pf11}
\end{equation}
where the state space is $X=[x_1,\cdots,x_5]^T= [\alpha,\dot{\alpha},p,\theta_1,\theta_2]^T$ and $\dot{x}_1 = x_2,\dot{x}_4 = 0 , \dot{x}_5 = 0$ which is a random walk. Eq. (\ref{eq:pf11}) is a nonlinear and hybrid model due to the relation of $x_3$ and the noise is reflected in the measurement sensor $(y=[p,\dot{p},u_c,u_d]^T+\nu)$.
For PF, we need the state transition matrix which is the discretized version of (\ref{eq:pf11}), however, due to complexity, we need to further simplify the computation.

\subsubsection{Quantitative analysis}
The data sample time is $T=1$ ms and each data pack includes 100 samples which results in an algorithm sample time equal to $T_s=100$ ms. The running time of the RLS approach is approximately $2$ ms, however, for PF and ARMCMC it is limited by 10 ms, and 80 ms, respectively. Point estimation is obtained from solution of the Maximum A-Posteriori (MAP). In the ARMCMC, we suggest the mode at the modification phase and the mean during the reinforcement phase; this estimate is denoted as ARMCMC Mode and Average Point estimate (AR-MAP). The true parameters with the means and standard deviations of the posterior probability using AR-MAP are shown in Table \ref{tab:par}. The parameter point estimate results are shown in
 Fig. \ref{fig:teta_2nd}. The estimation errors for $\theta_1,\theta_2$, and prediction of angles ($\Delta \theta_1,\Delta \theta_2,\Delta \alpha$) for three recursive approaches are illustrated in Table \ref{tab:comp}. As observed, AR-MAP offers effective parameter tracking while PF and RLS failed to converge to the ground truth, even when it remains constant. The reason for this is the inconsistency of the model when transitioning between retraction and contraction. However, the inference time for RLS and PF is relatively lower. To compare the actual signal of interest, we plotted the estimation of the angle in Figure \ref{fig:alfa_2nd}, which demonstrates the superiority of the final angle prediction for the AR-MAP method.

 
 \begin{figure}[t!]
\begin{center}
\centerline{\includegraphics[width=1\linewidth,trim={2cm  0.6cm 1.2cm .8cm },clip]{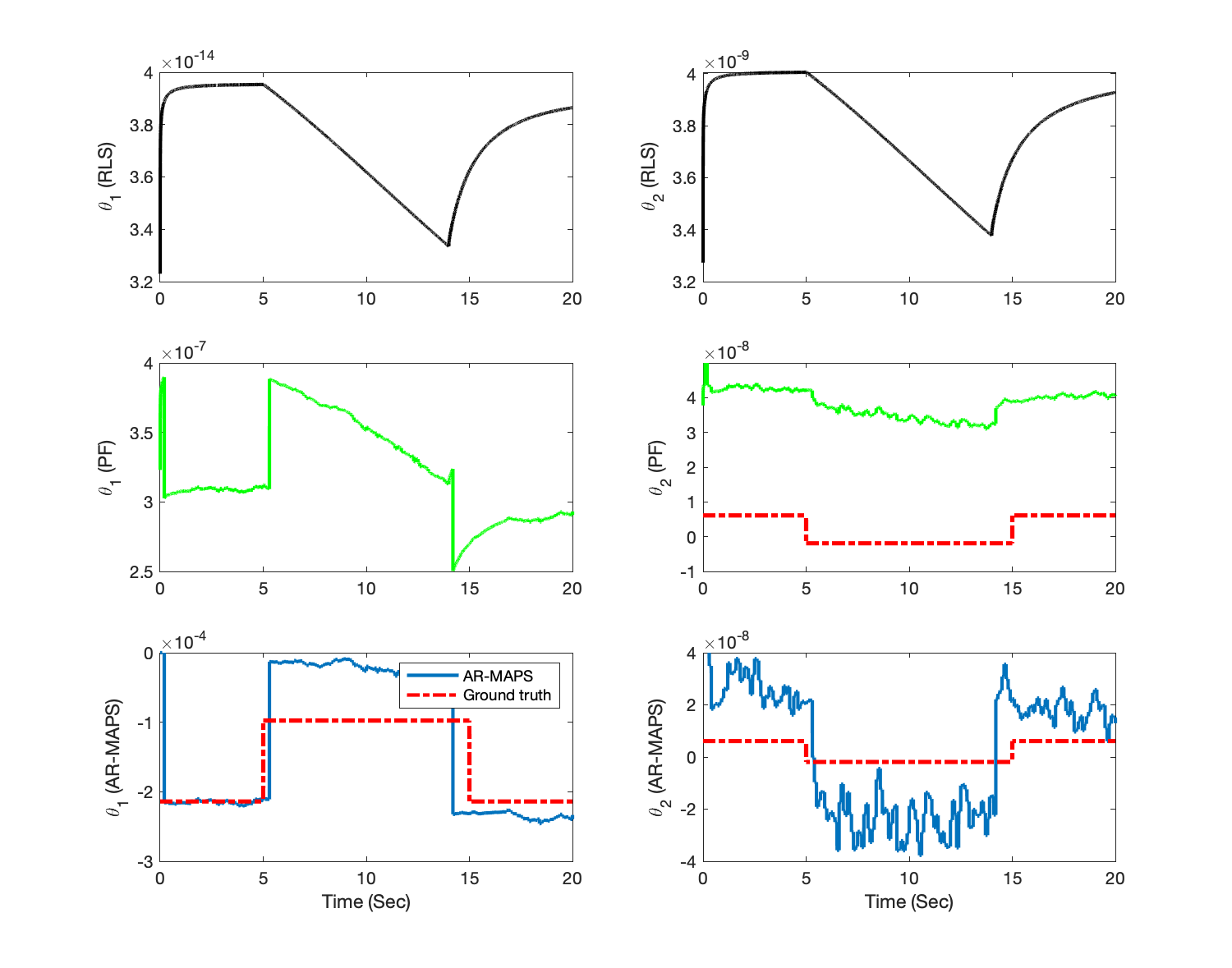}}
  \caption{Parameter variation for RLS, AR-MAP, and PF }
  \label{fig:teta_2nd}
  \end{center}
\vskip -0.2in
\end{figure}
\begin{figure}
\begin{center}
\centerline{\includegraphics[width=0.8\linewidth,trim={1cm  0.4cm 1cm .8cm },clip]{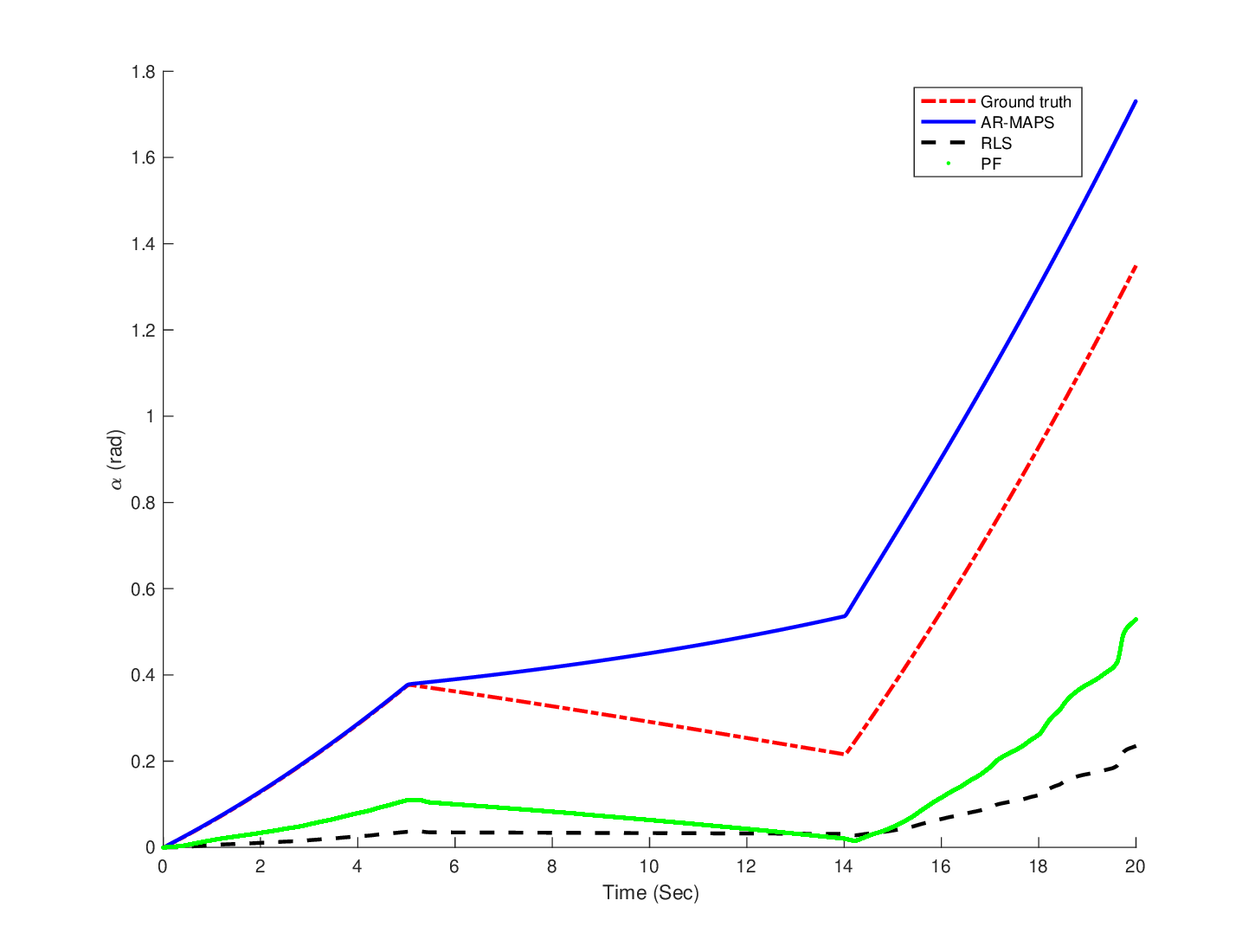}}
 \caption{Angle of the actuator comparison of RLS and AR-MAP for soft bending actuator. }
  \label{fig:alfa_2nd}
    \end{center}
\vskip -0.2in
\end{figure}

\begin{table}[!t]
\caption{Summery of parameter estimation for AR-MAP}
\vskip 0.15in
\begin{center}
\begin{small}
\begin{sc}
\begin{tabular}{lcccr}
\toprule
Parameter &  $q_4$&  $q_5$ &  $q_6$ & $q_7$  \\
Scale &  $10^{-4}$&  $10^{-9}$ &  $ 10^{-5}$ &  $10^{-9}$  \\
\midrule
True  &  $-2.14$ & $6.12 $ & $-9.76 $ &$-1.90 $\\
Posterior mean   &$-1.89$& $1.14$ & $-11.43$ &$-0.73$\\
Posterior s.d.  &$0.0041$& $0.21$ & $0.0034$&$ 0.31$\\
\bottomrule
\end{tabular}
\end{sc}
\end{small}
\end{center}
\vskip -0.1in
\label{tab:par}
\end{table}

\begin{table}[!t]
\caption{Comparison of RLS, PF, and AR-MAP for fluid soft bending actuator}
\vskip 0.15in
\begin{center}
\begin{small}
\begin{sc}
\begin{tabular}{lccr}
\toprule
Metric & RLS & PF & AR-MAP \\
\midrule
$\|\Delta\theta_1\|$ &$0.0235$&$0.0233$&$0.0089$\\
$\|\Delta \theta_2\|\times 10^{-6}$&$0.60053$&$4.8517$&$2.6951$ \\
$\|\Delta \alpha\|$&$61.4862 $&$50.6386$ &$32.4846$\\
 Running time&$2$ms&$10$ms& $80$ms\\
\bottomrule
\end{tabular}
\end{sc}
\end{small}
\end{center}
\vskip -0.1in
\label{tab:comp}
\end{table}



\subsection{Hunt-Crossley model}
In this section, we demonstrate ARMCMC by identifying parameters of the Hunt-Crossley model, which represent an environment involving a needle contacting soft material. The needle is mounted as an end-effector on a high-precision robotic arm, which switches between two modes: free motion and contact.%
For medical applications, effectively mounting force sensors on the needle is typically infeasible due to sterilizing issues. 
Due to abrupt changes in the model parameters when the contact is established or lost, online estimation of the force is extremely challenging.

\subsubsection{Contact dynamic model}
Consider the dynamics of contact as described by the Hunt-Crossley model, which is more consistent with the physics of contact than classical linear models such as Kelvin-Voigt \cite{haddadi2012real}. In order to overcome the shortcomings of linear models, \cite{hunt1975coefficient} proposed the following hybrid nonlinear model:
\begin{equation}
f_e(x(t))=\Bigg\{\begin{matrix}
K_ex^{p}(t)+B_ex^{p}(t)\dot{x}(t) & x(t)\geq 0\\
0& x(t)<0
\end{matrix},
\label{eq:huntt}
\end{equation}
in which $K_e, B_ex^p$  denote the nonlinear elastic and viscous force
coefficients, respectively. The parameter $p$ is typically between
$1$ and $2$, depending on the material and the geometric properties
of contact. Also, $x(t), \dot{x}(t), f_e$ are the current position, velocity (as input $X$), and contact force (as output $Y$) in Eq. (\ref{eq:bays44444444444444}). If $x \geq 0$, then the needle is  inside the soft material. 
$K_e, B_e, p$ are three unknown parameters ($\theta$ in Eq. (\ref{eq:bays44444444444444})) that need to be estimated. 
\pa{
A nonlinear hybrid model based on a reality-based soft environment is considered \cite{abolhassani2007needle}.
To relax the nonlinearity in the parameters, we use log:}

\begin{equation}\begin{split}
\log(f_e)&=\log(K_ex_s^p+B_e\dot{x}_sx_s^p),\\
\log(f_e)=&p\log(x_s)+\log(K_e+B_e\dot{x}_s).
\label{eq:hnt}
\end{split}\end{equation}
For RLS, we also need to further simplify the relation by assuming $B_e/K_e \dot{x}_s<<1$. Note that the vector of parameters $(\theta)$ are not independent, which may lead to divergence. 

\begin{equation}\begin{split}
\log(1+B_e/K_e\dot{x}_s)&\approx  B_e/K_e\dot{x}_s,\\
\log(f_e)=p\log(x_s)+&\log(K_e)+B_e/K_e\dot{x}_s.
\label{eq:hnt3}
\end{split}\end{equation}
\begin{equation}\begin{split}
U=&[1,\dot{x}_s,\log(x_s)],\\
\theta=&[\log(K_e),B_e/K_e,p]^T.
\label{eq:hnt44}
\end{split}\end{equation}

\subsubsection{Setup}

The data structure is the same as the previous simulation.  Prior distribution of all three parameters ($K_e, B_e, p$) are initialized to $N(1,0.1)$ (a normal distribution with
$\mu=1$ and $\sigma=0.1$) as a non-informative initial guess, whereas after applying the control effort and availability of the first data pack, these distributions are updated.  Moreover, as more data is collected, the spread of the posterior pdf decreases. After around 5 seconds, the needle goes outside of the soft material, and has zero force; this is equivalent to all parameters being set to zero. 
The color-based visualization of the probability distribution over time is used for the three parameters in Fig. \ref{fig:teta_pMCMC}. During the period of time that the entire space is blue (zero probability density), there is no contact and the parameter values are equal to zero.

Since we are taking a Bayesian approach, we are able to estimate the entire posterior pdf. 
However, for the sake of illustration, the point estimates are computed from the ARMCMC algorithm by using AR-MAP method. The results are shown in Fig. \ref{fig:teta_MCMC} for the time-varying parameters $\theta_1=K_e, \theta_2=B_e,\theta_3=p$. During the times that RLS results are chattering due to the use of saturation (if not, the results would have diverged), the needle is transitioning from being inside the soft material to the outside or vice versa. In addition, due to the assumption (\ref{eq:hnt3}), performance can decrease even when there is no mode transition. 
Although in the RLS approach estimated parameters suddenly diverged during free motion as the regression vectors are linearly dependent,  with ARMCMC this is not an issue. There is a delay in estimation of ARMCMC which is higher when changing from a contact to a free space. The reason is the existing noise in the measurement which requires the MCMC estimator to have a higher threshold ($\zeta_{\text{th}}$).  The result of  force prediction is presented in Fig. \ref{fig:fe_prd_MCMC}, which shows the effect of two different identification approaches. This probability of interest can be easily obtained by deriving the parameter density at one’s disposal. As we can see, The RLS, and MCMC suffer from numerous high frequency responses due to the level of uncertainty and identification error. 

\subsubsection{Quantitative comparison}
Quantitative details of comparing a naive point estimate of the ARMCMC algorithm by averaging the samples, ARMCMC Averaging Point estimate (AR-AP), and the RLS method are listed in Table \ref{tab:ver}. This reveals an improvement of over 70\%  in the precision of all model parameters throughout the study by using the Mean Absolute Error (MAE) criteria. The force estimation error also had more than 55\% improvement. Among parameters, the viscose ($B_e$) has the largest error in the RLS method since it is underestimated due to the restrictive assumption in Eq. (\ref{eq:hnt3}).  The AR-MAP approach uplifts the performance of the parameter identification and the force estimation. 
The ARMCMC algorithm paves the way for successful online identification of model parameters while restrictive presumptions are neglectable. 
Meanwhile, by employing the strength in the Bayesian paradigm,  this powerful tool can cooperate different sources of knowledge.
The recursion in the Bayesian optimization framework provides a suboptimal solution for the identification of nonlinear, hybrid, non-Gaussian systems. We also compare ARMCMC to MCMC. 
For real-time implementation, MCMC requires more time to converge. In this example, with $\lambda=0.7$, the value  of $k_{min}$ is $15000$ for MCMC but only $6000$ for ARMCMC with $\epsilon=0.01,\delta=0.9$. Two possible ways to address this drawback in MCMC are to reduce the number of samples to 5000 per algorithm iteration (denoted MCMC-1 in Table \ref{tab:ver}), which results in worse precision and reliability compared to ARMCMC, or to increase the algorithm sample time to 0.2 (denoted MCMC-2 in Table \ref{tab:ver})  which would cause more delay in the estimation result and slower responses to changes in the parameter. 


 \begin{figure}[t!]
\begin{center}
\centerline{\includegraphics[width=1\linewidth,trim={1cm  0.4cm 1cm 1cm },clip]{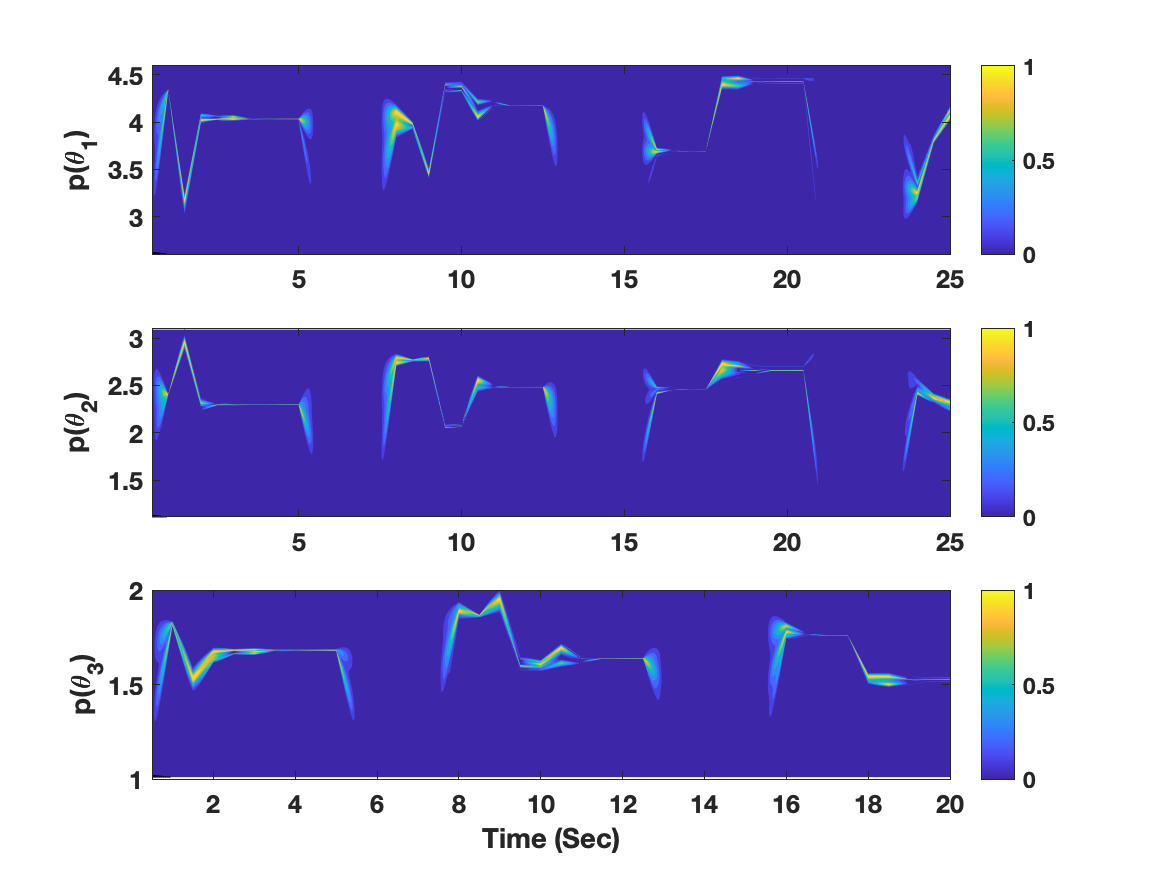}}
  \caption{Probability distribution of parameters ($\theta_1=K_e, \theta_2=B_e,\theta_3=p$)  using ARMCMC.}
  \label{fig:teta_pMCMC}
\end{center}
\vskip -0.2in
\end{figure}



\begin{figure}[t!]
\begin{center}
\centerline{\includegraphics[width=1\linewidth,trim={.8cm  .4cm 1cm 0.6cm },clip]{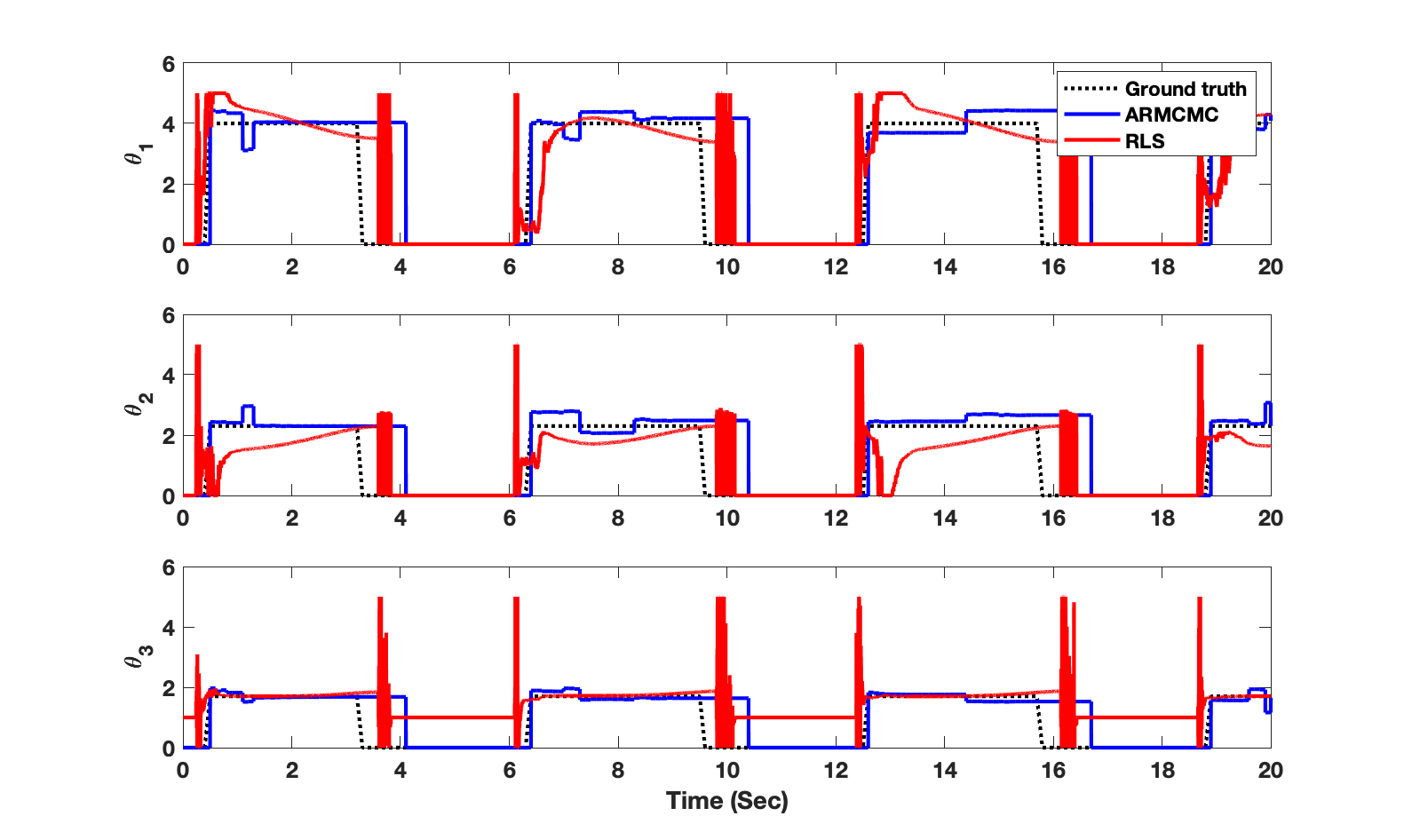}}
 \caption{Model parameters ($\theta_1=K_e, \theta_2=B_e,\theta_3=p$)  point estimation in AR-MAP.}
 \label{fig:teta_MCMC}
\end{center}
\vskip -0.2in
\end{figure}

\begin{figure}[t!]
\begin{center}
\centerline{\includegraphics[width=.8\linewidth,trim={1cm  0.3cm 0.5cm 0.5cm },clip]{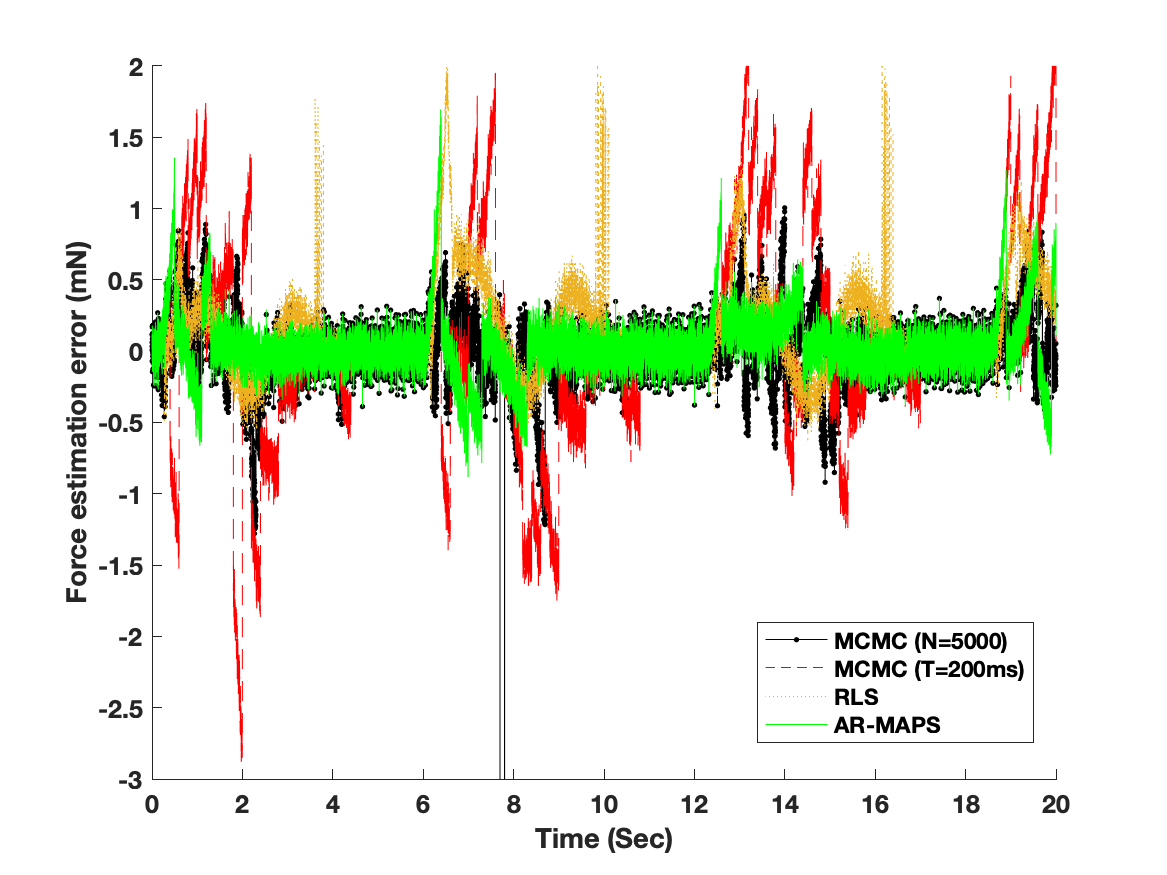}}
\caption{Force prediction error profile. }
\label{fig:fe_prd_MCMC}
\end{center}
\vskip -0.2in
\end{figure}


\begin{table}[!t]
\caption{Comparison of RLS and point estimate of ARMCMC and MCMC for environment identification.}
\vskip 0.15in
\begin{center}
\begin{small}
\begin{sc}
\begin{tabular}{lcccr}
\toprule
Errors (MAE) &  $K_e$&  $B_e$ &  $p$ & $F_e$  (mN)  \\
\midrule
RLS  &  0.5793 & 0.9642 & 0.3124&51.745\\
MCMC-1    &0.6846& 0.8392 & 0.3783 &76.695\\
MCMC-2  &0.7294& 0.9964 & 0.4195& 101.88\\
AR-AP      &0.0774 & 0.0347 & 0.0945&33.774\\
AR-MAP   & 0.0617 & 0.0316 & 0.0756& 31.659\\
\bottomrule
\end{tabular}
\end{sc}
\end{small}
\end{center}
\vskip -0.1in
\label{tab:ver}
\end{table}

\section{CONCLUSION}
\label{sec:con}
This paper presented an adaptive recursive MCMC algorithm for online identification of model parameters with full probability distribution. When applied to systems involving abrupt changes of model parameters, conventional approaches suffer from low performance. 
Results on the Hunt-Crossley and fluid soft bending actuator as nonlinear hybrid dynamic models was compared with well-known conventional identification approaches.
The proposed method adapted quickly to abrupt changes and relaxes the prerequisite conditions in the parameters. \pa{ Since ARMCMC is a more general approach, it can be applied to a wider range of systems. In the case of single-mode nonlinear systems with  LIP condition, one may use RLS which provides the best result 40x faster compared to ARMCMC. However,  for hybrid/multi-modal systems with fairly accurate parametric models, ARMCMC can output reasonable online inferences. They can be fairly accurate since the uncertainty in the models can be captured in the probability distribution of the parameters. However, this option is not available for conventional approaches like RLS or KF. For PF as another sampling-based methods for online Bayesian optimization, it relies on  the dynamical model of evolution to enhance the acceptance of the proposal distribution. Nonetheless, in the case of parameter identification, as the evolution model for parameters is a random walk, there is no extra knowledge for the estimator. PF as a sequential Monte Carlo (SMC) method assumes a connectivity in updating parameters in the resampling phase, therefore, ARMCMC surpasses PF performance in parameter identification problems.}

\bibliographystyle{IEEEtran} 
\bibliography{myref}

\end{document}